\documentclass[reqno]{amsart}

\usepackage{amssymb}
\usepackage{subfigure, graphicx}

\DeclareMathOperator{\rank}{rank}
\DeclareMathOperator{\sign}{sign}
\DeclareMathOperator{\Syl}{Syl}
\DeclareMathOperator{\Tr}{Tr}

\theoremstyle{plain}
\newtheorem{theorem}{Theorem}
\newtheorem{lemma}{Lemma}

\theoremstyle{definition}
\newtheorem{problem}{Problem}

\theoremstyle{remark}
\newtheorem{remark}{Remark}

\begin{document}

\title[A Solution to the 4P3V Pose Problem in Case of Collinear Cameras]{A Non-Iterative Solution to the Four-Point Three-Views Pose Problem in Case of Collinear Cameras}

\author{E.V. Martyushev}

\address{South Ural State University, 76 Lenin Avenue, Chelyabinsk 454080, Russia}
\email{mev@susu.ac.ru}

\begin{abstract}
We give a non-iterative solution to a particular case of the four-point three-views pose problem when three camera centers are collinear. Using the well-known Cayley representation of orthogonal matrices, we derive from the epipolar constraints a system of three polynomial equations in three variables. The eliminant of that system is a multiple of a 36th degree univariate polynomial. The true (unique) solution to the problem can be expressed in terms of one of real roots of that polynomial. Experiments on synthetic data confirm that our method is robust enough even in case of planar configurations.
\end{abstract}

\keywords{Four-Point Three-Views Pose Problem, Structure-from-Motion, Epipolar Constraints, Cayley Representation}

\maketitle

\section{Introduction}

We first recall some definitions from multiview geometry and formulate the problem in question, see~\cite{HZ, Maybank} for details. A \textit{pinhole camera} is a triple $(O, \Pi, p)$, where $\Pi$ is an image plane, $p$ is a central projection of points in 3-dimensional Euclidean space onto~$\Pi$, and $O$ is a camera center (center of the projection~$p$). The \textit{focal length} is the distance between $O$ and~$\Pi$, the orthogonal projection of~$O$ onto~$\Pi$ is called the \textit{principal point}. A pinhole camera is called \textit{calibrated} if all its intrinsic parameters (such as focal length and principal point's coordinates) are known.

\begin{problem}[The four-point three-views pose problem]
\label{problem1}
Let us consider three calibrated pinhole cameras with centers $O_1$, $O_2$, $O_3$ and four scene points $P_1$, \ldots, $P_4$ being in front of the cameras in 3-dimensional Euclidean space, see Figure~\ref{fig:P4O3}. In every local coordinate system $O_jxyz$, associated with the $j$th camera, the homogeneous coordinates $\begin{pmatrix}x_{ji} & y_{ji} & 1\end{pmatrix}^{\mathrm T}$ of all $P_i$'s are only known. One must find the relative position and orientation of the second and third cameras with respect to the first one and the coordinates of all $P_i$'s in $O_1xyz$.
\end{problem}

\begin{figure}[ht]
\centering
\includegraphics[width=0.6\hsize]{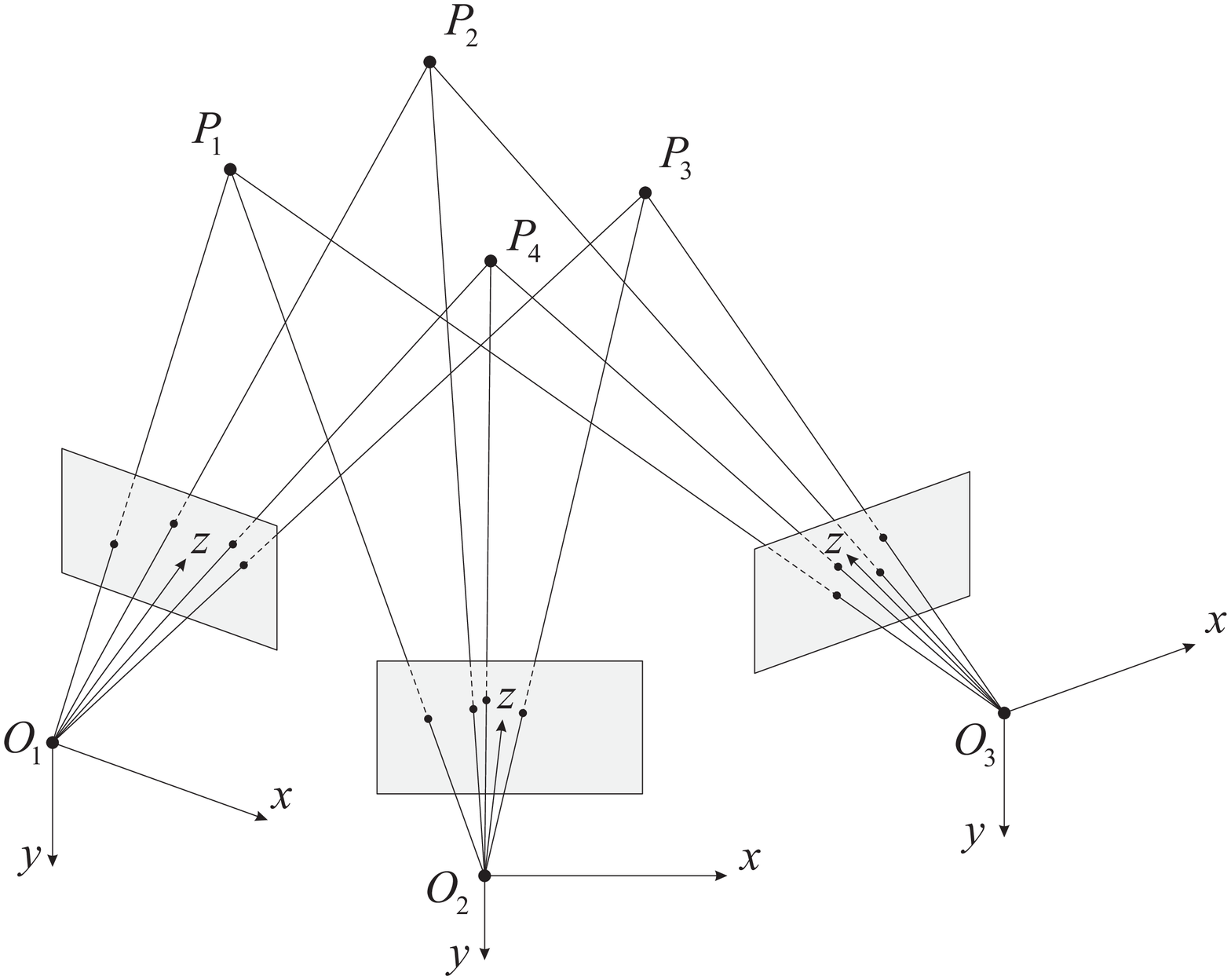}
\caption{To formulation of the four-point three-views pose problem}\label{fig:P4O3}
\end{figure}

\begin{remark}
\label{rem:ambiguity}
It is well-known that the solution to Problem~\ref{problem1} is only defined up to an overall scale. That is, multiplying the coordinates of all the reconstructed points by any $\lambda > 0$, we get another solution to Problem~\ref{problem1}. In order to resolve this ambiguity, we assume from now on that the length of the \textit{baseline} joining $O_1$ and $O_2$ is given and equals~$d$.
\end{remark}

Problem~\ref{problem1} is known to be slightly over-determined and in general to have a unique solution~\cite{HN, QTM}. The only iterative approach to the problem has been proposed in~\cite{NS}.

In the paper presented, we consider a light (but still highly nontrivial) version of Problem~\ref{problem1} when the three camera centers $O_1$, $O_2$ and $O_3$ are collinear. We use the method of paper~\cite{Mart} to derive a relatively simple system of three polynomials in three variables. It is shown that the eliminant of that system is a multiple of a 36th degree irreducible univariate polynomial. It is conjectured that in general one of its real roots encodes a solution to Problem~\ref{problem1}. Experiments show that our approach is comparable in accuracy with the existing five-point solvers.

Here we outline our algorithm:
\begin{enumerate}
\item
Transform the initial data $x_{ji}$, $y_{ji}$ in such a way that $x_{j1} = y_{j1} = x_{j2} = 0$, $j=1,2,3$, applying rotations~\eqref{eq:rotations}.
\item
Derive from the epipolar constraints six polynomials $f^{(j)}_i$ by~\eqref{eq:systemfji} and one polynomial $f^\text{mix}$ by~\eqref{eq:fmix}.
\item
Construct system~\eqref{eq:h2h3hmix} of three polynomials in three variables.
\item
Using resultants and polynomial divisions, derive a 36th degree univariate polynomial $\mathcal S$.
\item
For every real root of~$\mathcal S$, find the matrices $R^{(2)}$, $R^{(3)}$ and~$[t]_\times$ by~\eqref{eq:cayley} and~\eqref{eq:transl} respectively.
\item
Find the coordinates of $O_2$, $O_3$ and $P_i$ by~\eqref{eq:coordsO2},~\eqref{eq:coordsO3} and~\eqref{eq:coordsPi} respectively.
\end{enumerate}

The paper is organized as follows. In Section~\ref{sec:description}, we describe in detail our algorithm. In Section~\ref{sec:experiments}, we make a comparison of our algorithm with the five-point Li-Hartley solver~\cite{LH} on two sets of synthetic data. In Section~\ref{sec:discussion}, we discuss the results of the paper.

\section{Description of the algorithm}
\label{sec:description}

\subsection{Transformation of the initial data}

Initial data for our algorithm are the values $x_{ji}$, $y_{ji}$ for $j = 1, 2, 3,$ and $i = 1, \ldots, 4$, as well as the distance~$d$ (see Remark~\ref{rem:ambiguity}).

Without loss of generality, we can set $x_{j1} = y_{j1} = x_{j2} = 0$ for $j = 1, 2, 3$. If $x_{j1}$, $y_{j1}$ and $x_{j2}$ differ from zero, we rotate $O_jxyz$ to $O_j\tilde{x}\tilde{y}\tilde{z}$ with the matrix
\begin{equation}
\label{eq:rotations}
\rho_j = R_{\varphi_j} R_{\theta_j} R_{\psi_j},
\end{equation}
where $R_{\varphi_j}$, $R_{\theta_j}$, $R_{\psi_j}$ are rotations through the angles $\varphi_j$, $\theta_j$, $\psi_j$ respectively around the axes $O_jz$, $O_jy$, $O_jx$ respectively. The angles can be expressed by the formulas:
\begin{equation}
\begin{split}
\tan\psi_j &= y_{j1},\\
\tan\theta_j &= \frac{x_{j1}}{\sqrt{1 + y_{j1}^2}},\\
\tan\varphi_j &= \frac{x_{j2}(1 + y_{j1}^2) - x_{j1}(1 + y_{j1}y_{j2})}{(y_{j2} - y_{j1})\sqrt{1 + x_{j1}^2 + y_{j1}^2}}.
\end{split}
\end{equation}
One then verifies that $\tilde{x}_{j1} = \tilde{y}_{j1} = \tilde{x}_{j2} = 0$. Besides, the depth of~$P_1$ from the camera center~$O_j$ is
\[
\sqrt{1 + x_{j1}^2+y_{j1}^2} > 0,
\]
i.e. the point $P_1$ remains to be in front of all the cameras.

We will see that the above transformation of the initial data, being quite simple, noticeably simplify our further computations.

\subsection{Epipolar constraints and essential matrices}
From now on we assume that the three camera centers $O_1$, $O_2$ and~$O_3$ are collinear.

Denote by $\begin{pmatrix}x'_{ji} & y'_{ji} & 1\end{pmatrix}^{\mathrm T}$ and $\mathbf{t} = \begin{pmatrix}t_x & t_y & t_z\end{pmatrix}^{\mathrm T}$ the \textit{directing vectors} of $O_jP_i$ and $O_1O_2$ respectively in the \textit{world} coordinate system $O_1xyz$. It follows that the directing vector of $O_1O_3$ is $\sigma \mathbf{t}$, where $\sigma$ is either $+1$ or~$-1$.

Let Cartesian coordinates of $P_i$ and $O_j$ be $(x_{P_i}, y_{P_i}, z_{P_i})$ and $(x_{O_j}, y_{O_j}, z_{O_j})$ respectively. Then, it is easy to see that ($j = 2, 3$; $i = 1, \ldots, 4$)
\begin{equation}
\begin{split}
(x_{1i} - x'_{ji})\,t_z z_{P_i} &= (t_x - t_z x'_{ji})\,z_{O_j},\\
(y_{1i} - y'_{ji})\,t_z z_{P_i} &= (t_y - t_z y'_{ji})\,z_{O_j}.
\end{split}
\end{equation}
Eliminating $z_{P_i}$, $z_{O_2}$ and $z_{O_3}$ from this system, we get
\begin{equation}
\label{eq:system}
\begin{split}
t_x (y'_{2i} - y_{1i}) + t_y (x_{1i} - x'_{2i}) + t_z (y_{1i}x'_{2i} - x_{1i}y'_{2i}) &= 0,\\
t_x (y'_{3i} - y_{1i}) + t_y (x_{1i} - x'_{3i}) + t_z (y_{1i}x'_{3i} - x_{1i}y'_{3i}) &= 0.
\end{split}
\end{equation}

Consider matrices $R^{(j)} \in \mathrm{SO}(3)$, $j = 2, 3$, rotating $O_jxyz$ so that
\begin{equation}
\label{eq:rotation}
\omega_{ji}\begin{pmatrix}x'_{ji} \\ y'_{ji} \\ 1\end{pmatrix} = R^{(j)} \begin{pmatrix}x_{ji} \\ y_{ji} \\ 1\end{pmatrix}.
\end{equation}
where $\omega_{ji}$ is called the \textit{depth} of the point $P_i$ from the camera center~$O_j$. Then, it is easy to see that~\eqref{eq:system} is equivalent to the so-called \textit{epipolar constraints}:
\begin{equation}
\label{eq:epipolar}
\begin{pmatrix}
x_{1i} & y_{1i} & 1
\end{pmatrix} E^{(j)} \begin{pmatrix} x_{2i} \\ y_{2i} \\ 1 \end{pmatrix} = 0,
\end{equation}
where $E^{(j)} = [t]_\times R^{(j)}$ is called the \textit{essential matrix} and $[t]_\times$ is the skew-symmetric cross-product operator.

\subsection{Seven polynomials in six variables}
Our approach is based on the following well-known result.
\begin{theorem}[\cite{Cayley}]
\label{thm:cayley}
If a matrix $R \in \mathrm{SO}(3)$ is not a rotation through the angle~$\pi$ about certain axis, then $R$ can be represented as
\begin{equation}
\label{eq:cayley}
\frac{1}{\Delta}\begin{pmatrix} 1-u^2+w^2(1-v^2) & 2(vw^2-u) & 2w(u+v) \\ 2(vw^2+u) & 1-u^2-w^2(1-v^2) & 2w(uv-1) \\ 2w(u-v) & 2w(uv+1) & 1+u^2-w^2(1+v^2) \end{pmatrix},
\end{equation}
where $\Delta = 1+u^2+w^2(1+v^2)$.
\end{theorem}

Let $Oxyz$ and $OXYZ$ be two coordinate systems with a common origin and of the same handedness. Then the \textit{Euler angles}, transforming $Oxyz$ to $OXYZ$, are defined as~\cite{LL}:
\begin{itemize}
\item
$\varphi$ is the angle between the $x$-axis and the \textit{line of nodes}, i.e. the line of intersection of the $xy$ and the $XY$ coordinate planes.
\item
$\theta$ is the angle between the $z$-axis and the $Z$-axis.
\item
$\psi$ is the angle between the line of nodes and the $X$-axis.
\end{itemize}

\begin{remark}
By direct calculation one verifies that the Euler angles $\varphi$, $\theta$ and~$\psi$ can be expressed in terms of the parameters $u$, $v$ and~$w$ as follows:
\begin{equation}
\label{eq:euler}
\begin{split}
\varphi &= \arctan\frac{u + v}{1 - uv},\\
\psi &= \arctan\frac{u - v}{1 + uv},\\
\theta &= \pm 2\arctan w\sqrt{\frac{1 + v^2}{1 + u^2}}.
\end{split}
\end{equation}
\end{remark}

\begin{lemma}
\label{lem:identity}
Let $\varphi_j$, $\theta_j$ and~$\psi_j$ be Euler angles  parameterizing the matrix~$R^{(j)}$. In case of collinear camera centers $O_1$, $O_2$ and $O_3$, the following identity holds:
\[
\tan\varphi_2 = \tan\varphi_3.
\]
\end{lemma}

\begin{proof}
Due to the condition $x_{j1} = y_{j1} = 0$, $j = 1, 2, 3$, the $z$-axes of all $O_jxyz$ intersect in the only point~$P_1$. Since $O_1$, $O_2$ and $O_3$ are collinear, the lines of nodes of $O_2xyz$ and $O_3xyz$ are identical and hence, by definition, $\varphi_2 = \varphi_3 \mod \pi$.
\end{proof}

\begin{lemma}
\label{lem:invtrans}
Let $E^{(j)}(u, v, w)$ be an essential matrix subject to system~\eqref{eq:epipolar}. If
\begin{equation}
\label{eq:invtrans}
\begin{split}
u'_j &= -u_j^{-1},\\
v'_j &= -v_j^{-1},\\
w'_j &= -\frac{v_j}{u_j}\, \frac{y_{12}w_j(v_j + u_j) + y_{j2}w_j(v_j - u_j) - 2y_{12}y_{j2}u_j}{y_{12}(v_j + u_j) - y_{j2}(v_j - u_j) + 2y_{12}y_{j2}v_jw_j},
\end{split}
\end{equation}
then $E^{(j)}(u'_j, v'_j, w'_j) = -E^{(j)}(u_j, v_j, w_j)$.
\end{lemma}

\begin{proof}
Let $E^{(j)} = [t]_\times R^{(j)}$ be subject to~\eqref{eq:epipolar}. Then, from the first two equations of~\eqref{eq:epipolar} (for $i = 1, 2$) we find
\begin{equation}
\label{eq:transl}
\begin{split}
t_x &= t_z\, \frac{R^{(j)}_{13}(R^{(j)}_{12}y_{j2} + R^{(j)}_{13})y_{12}}{(R^{(j)}_{23}R^{(j)}_{12} - R^{(j)}_{13}R^{(j)}_{22})y_{j2} + R^{(j)}_{13}(R^{(j)}_{32}y_{j2} + R^{(j)}_{33})y_{12}},\\
t_y &= t_x\, \frac{R^{(j)}_{23}}{R^{(j)}_{13}} = t_x\, \frac{u_jv_j - 1}{u_j + v_j},
\end{split}
\end{equation}
where $R^{(j)}_{ik}$ is the $(i, k)$th entry of~$R^{(j)}$. Substituting these values into the matrix~$[t]_\times$, we get $E^{(j)} = E^{(j)}(u_j, v_j, w_j)$. By a straightforward computation, we find that the equation $E^{(j)}(u_j, v_j, w_j) + E^{(j)}(u'_j, v'_j, w'_j) = 0$ has the only solution~\eqref{eq:invtrans}.
\end{proof}

Since system~\eqref{eq:system} is linear and homogeneous in $t_x$, $t_y$ and $t_z$, it has a nontrivial solution if and only if $\rank \begin{pmatrix}Q^{(2)} \\ Q^{(3)}\end{pmatrix} \leq 2$, where
\begin{equation}
\label{eq:matrixQj}
Q^{(j)} = \begin{pmatrix}
y'_{j1} & -x'_{j1} & 0 \\
y'_{j2} - y_{12} & -x'_{j2} & y_{12}x'_{j2} \\
y'_{j3} - y_{13} & x_{13} - x'_{j3} & y_{13}x'_{j3} - x_{13}y'_{j3} \\
y'_{j4} - y_{14} & x_{14} - x'_{j4} & y_{14}x'_{j4} - x_{14}y'_{j4}
\end{pmatrix}.
\end{equation}
Let $i = i_{(\alpha, \beta, \gamma)}$ number triples from $\{1, 2, 3, 4\}$ so that $i_{(1, 2, 3)} = 1$, $i_{(1, 2, 4)} = 2$, $i_{(1, 3, 4)} = 3$. Then, it follows that $\det Q^{(j)}_i = 0$, where $Q^{(j)}_i$ is a $3\times 3$ submatrix of~$Q^{(j)}$ corresponding to the rows $\alpha$, $\beta$ and $\gamma$. Thus, we get a system
\begin{equation}
\label{eq:systemfji}
f^{(j)}_i = \det Q^{(j)}_i\,\frac{\Delta_j^2\, \omega_{j \alpha}\omega_{j \beta}\omega_{j \gamma}}{w_j} = 0, \quad j = 2, 3, \quad i = 1, 2, 3.
\end{equation}
where $\Delta_j = 1+u_j^2+w_j^2(1+v_j^2)$ and, for instance, $\omega_{j \alpha}$ is the depth of the point~$P_\alpha$ from~$O_j$ (cf.~\eqref{eq:rotation}). One verifies that in general each $f^{(j)}_i$ is an irreducible polynomial in $u_j$, $v_j$, $w_j$ of the sixth total degree.

\begin{remark}
In~\eqref{eq:systemfji}, we assume that $w_j \neq 0$. No solution is lost here, since $w_j = 0$ means that the matrix $R^{(j)}$ is a rotation about the $z$-axis, which may only occur if $O_j = O_1$.
\end{remark}

In addition to~\eqref{eq:systemfji}, we can also obtain more polynomial equations by treating the rows of both $Q^{(2)}$ and~$Q^{(3)}$. For example, consider a submatrix $Q^\text{mix}$ of~$\begin{pmatrix}Q^{(2)} \\ Q^{(3)}\end{pmatrix}$ consisting of the first two rows of~$Q^{(2)}$ and the second row of~$Q^{(3)}$. Then,
\begin{equation}
\label{eq:fmix}
f^\text{mix} = \det Q^\text{mix}\, \frac{\Delta_2^2 \Delta_3 \, \omega_{21}\omega_{22}\omega_{32}}{x'_{21}} = 0,
\end{equation}
where we assume that $x'_{21} \neq 0$. In general, $f^\text{mix}$ is an irreducible polynomial in $u_2$, $v_2$, $w_2$, $u_3$, $v_3$, $w_3$ of the seventh total degree.

\subsection{Two polynomials in three variables}

In this subsection, we deal with the polynomials $f^{(j)}_i$ defined in~\eqref{eq:systemfji}. We are going to eliminate the variables $w_2$ and $w_3$ from them and then simplify the obtained polynomials using the transformations~\eqref{eq:invtrans}.

Notice that every polynomial $f^{(j)}_i$ is of the second degree in~$w_j$, i.e.
\begin{equation}
\label{eq:fjiform}
f^{(j)}_i = a^{(j)}_iw_j^2 + b^{(j)}_iw_j + c^{(j)}_i,
\end{equation}
where $a^{(j)}_i$, $b^{(j)}_i$, $c^{(j)}_i$ are polynomials in $u_j$ and~$v_j$. Consider a matrix
\[
F^{(j)} = \begin{pmatrix} a^{(j)}_1 & b^{(j)}_1 & c^{(j)}_1 \\ a^{(j)}_2 & b^{(j)}_2 & c^{(j)}_2 \\ a^{(j)}_3 & b^{(j)}_3 & c^{(j)}_3 \end{pmatrix}.
\]
Then, system~\eqref{eq:systemfji} (for each $j$) has a solution if and only if $\det F^{(j)} = 0$, i.e. we get two polynomial equations:
\begin{equation}
\label{eq:detFj}
g^{(j)}(u_j, v_j) = \det F^{(j)} = 0, \quad j = 2, 3.
\end{equation}
The polynomials $g^{(2)}$ and $g^{(3)}$ are of the 10th total degree in the variables $u_2$, $v_2$ and $u_3$, $v_3$ respectively. We can reduce the total number of variables to three by introducing a new variable
\[
s = \frac{u_2 + v_2}{1 - u_2v_2} = \frac{u_3 + v_3}{1 - u_3v_3},
\]
where the second equality holds due to Lemma~\ref{lem:identity}. Note that the variable~$s$ is unchanged under the transformations~\eqref{eq:invtrans}, i.e.
\[
s' = \frac{u'_j + v'_j}{1 - u'_jv'_j} = \frac{u_j + v_j}{1 - u_jv_j} = s.
\]

Substituting
\begin{equation}
\label{eq:vjformula}
v_j = \frac{s - u_j}{1 + su_j}.
\end{equation}
to~\eqref{eq:detFj}, we get
\[
h^{(j)}(u_j, s) = (1 + su_j)^6\, g^{(j)}(u_j, \frac{s - u_j}{1 + su_j}).
\]
By Lemma~\ref{lem:invtrans}, $h^{(j)}(u'_j, s') = h^{(j)}(-u_j^{-1}, s) = 0$. Moreover, $h^{(j)}$ has a special symmetric form:
\begin{equation}
\label{eq:gj}
h^{(j)} = u_j^{12} \sum\limits_{k=0}^{12} p^{(j)}_k \left[u_j^{12-k} + (-u_j)^{k-12}\right],
\end{equation}
where $p^{(j)}_k$ are 6th degree polynomials in~$s$. Due to the above symmetry, we can introduce a new variable $\tilde{u_j} = u_j - u_j^{-1}$, and then transform $h_j$ to the polynomial
\begin{equation}
\label{eq:tildehj}
\tilde{h}^{(j)} = \sum\limits_{k=0}^6 \tilde{p}^{(j)}_k \tilde{u}_j^k,
\end{equation}
where
\begin{equation}
\label{eq:tildep}
\begin{split}
\tilde{p}^{(j)}_{2k} &= \sum\limits_{i=k}^6 p^{(j)}_{12-2i} \varkappa_{i, k},\\
\tilde{p}^{(j)}_{2k+1} &= \sum\limits_{i=k}^5 p^{(j)}_{11-2i} \varkappa_{i+1/2, k+1/2},
\end{split}
\end{equation}
and $\varkappa_{i, k}$ denotes the sum of two binomial coefficients:
\begin{equation}
\label{eq:kappa}
\varkappa_{i, k} = \binom{i+k}{i-k} + \binom{i+k-1}{i-k-1}.
\end{equation}

As a result, we have two irreducible polynomials $\tilde{h}^{(2)}(\tilde{u}_2, s)$ and $\tilde{h}^{(3)}(\tilde{u}_3, s)$ of the 12th total degree each. In order to solve the problem, we need at least one more polynomial in the variables $\tilde{u}_2$, $\tilde{u}_3$ and~$s$.

\subsection{One more polynomial in three variables}

Let us consider the polynomial $f^\text{mix}$ defined in~\eqref{eq:fmix}. As in the previous subsection, we are going to eliminate the variables $w_2$ and~$w_3$ from $f^\text{mix}$ and then simplify the result using transformations~\eqref{eq:invtrans}.

First, we notice that $f^\text{mix}$ is of the second degree in the variable~$w_2$:
\begin{equation}
\label{eq:fmixform}
f^\text{mix} = a^\text{mix}w_2^2 + b^\text{mix}w_2 + c^\text{mix},
\end{equation}
where $a^\text{mix}$, $b^\text{mix}$, $c^\text{mix}$ are polynomials in the remaining five variables. Consider a matrix
\[
F^\text{mix} = \begin{pmatrix} a^{(2)}_1 & b^{(2)}_1 & c^{(2)}_1 \\ a^{(2)}_2 & b^{(2)}_2 & c^{(2)}_2 \\ a^\text{mix} & b^\text{mix} & c^\text{mix} \end{pmatrix}.
\]
The three polynomials $f^{(2)}_1$, $f^{(2)}_2$ and $f^\text{mix}$ have a common solution if and only if $\det F^\text{mix} = 0$. This yields
\begin{equation}
\label{eq:hatfmix}
\hat{f}^\text{mix} = \frac{\det F^\text{mix}}{(y_{12}^2 - y_{22}^2) u_2^2 + 2(y_{12}^2 + 2y_{12}^2y_{22}^2 + y_{22}^2)u_2v_2 + (y_{12}^2 - y_{22}^2) v_2^2} = 0.
\end{equation}
One verifies that the polynomial $\hat{f}^\text{mix}$ is in turn of the second degree in~$w_3$, i.e.
\begin{equation}
\label{eq:hatfmixform}
\hat{f}^\text{mix} = \hat{a}^\text{mix}w_3^2 + \hat{b}^\text{mix}w_3 + \hat{c}^\text{mix},
\end{equation}
where $\hat{a}^\text{mix}$, $\hat{b}^\text{mix}$, $\hat{c}^\text{mix}$ are polynomials in $u_2$, $v_2$, $u_3$ and $v_3$. Similarly, we define the matrix
\[
\hat{F}^\text{mix} = \begin{pmatrix} a^{(3)}_1 & b^{(3)}_1 & c^{(3)}_1 \\ a^{(3)}_2 & b^{(3)}_2 & c^{(3)}_2 \\ \hat{a}^\text{mix} & \hat{b}^\text{mix} & \hat{c}^\text{mix} \end{pmatrix},
\]
and from the constraint $\det \hat{F}^\text{mix} = 0$ obtain the polynomial
\begin{equation}
\label{eq:detFmix}
g^\text{mix} = \frac{\det \hat{F}^\text{mix}}{(y_{12}^2 - y_{32}^2) u_3^2 + 2(y_{12}^2 + 2y_{12}^2y_{32}^2 + y_{32}^2)u_3v_3 + (y_{12}^2 - y_{32}^2) v_3^2} = 0.
\end{equation}
Substituting~\eqref{eq:vjformula} to $g^\text{mix} = g^\text{mix}(u_2, v_2, u_3, v_3)$, we get
\begin{equation}
\label{eq:hmix}
h^\text{mix}(u_2, u_3, s) = \frac{(1 + su_2)^3 (1 + su_3)^3\, g^\text{mix}(u_2, \frac{s - u_2}{1 + su_2}, u_3, \frac{s - u_3}{1 + su_3})}{(1+u_2^2) (1+u_3^2) (x_{13} + y_{13}s) (x_{14} + y_{14}s)}.
\end{equation}

\begin{remark}
The denominators in~\eqref{eq:hatfmix},~\eqref{eq:detFmix} and~\eqref{eq:hmix} are supposed to be nonzero. By a straightforward computation, one verifies that in general their roots do not give a solution to Problem~\ref{problem1}.
\end{remark}

The polynomial $h^\text{mix}$ has a special symmetric form:
\begin{equation}
\label{eq:hmixform}
h^\text{mix} = u_2^4 u_3^4 \sum\limits_{k=0}^4 \sum\limits_{m=0}^4 p^\text{mix}_{k, m}\, [u_2^{4 - k} + (-u_2)^{k - 4}] [u_3^{4 - m} + (-u_3)^{m - 4}],
\end{equation}
where $p^\text{mix}_{k, m}$ are polynomials in~$s$. Substituting $\tilde{u_j} = u_j - u_j^{-1}$, we transform $h^\text{mix}$ to the polynomial
\begin{equation}
\label{eq:tildehmix}
\tilde{h}^\text{mix} = \sum\limits_{k=0}^2 \sum\limits_{m=0}^2 \tilde{p}^\text{mix}_{k, m}\, \tilde{u}_2^k\, \tilde{u}_3^m,
\end{equation}
where
\begin{equation}
\label{eq:tildepmix}
\begin{split}
\tilde{p}^\text{mix}_{2k, 2m} &= \sum\limits_{i=k}^2 \sum\limits_{j=m}^2 p^\text{mix}_{4-2i, 4-2j} \varkappa_{i, k} \varkappa_{j, m},\\
\tilde{p}^\text{mix}_{2k, 2m+1} &= \sum\limits_{i=k}^2 \sum\limits_{j=m}^1 p^\text{mix}_{4-2i, 3-2j} \varkappa_{i, k} \varkappa_{j+1/2, m+1/2},\\
\tilde{p}^\text{mix}_{2k+1, 2m} &= \sum\limits_{i=k}^1 \sum\limits_{j=m}^2 p^\text{mix}_{3-2i, 4-2j} \varkappa_{i+1/2, k+1/2} \varkappa_{j, m},\\
\tilde{p}^\text{mix}_{2k+1, 2m+1} &= \sum\limits_{i=k}^1 \sum\limits_{j=m}^1 p^\text{mix}_{3-2i, 3-2j} \varkappa_{i+1/2, k+1/2} \varkappa_{j+1/2, m+1/2},
\end{split}
\end{equation}
and $\varkappa_{i, k}$ is given by~\eqref{eq:kappa}.

\subsection{Some notions from elimination theory}
In this subsection, we briefly recall some notions from elimination theory, see~\cite{CLS} for details.

Given an ideal $J = \langle f_1,\ldots, f_m \rangle \subset \mathbb C[x_1, \ldots, x_n]$, the $l$th \textit{elimination ideal} $J_l$ is defined by $J_l = J \cap \mathbb C[x_{l+1}, \ldots, x_n]$, i.e. $J_l$ consists of all consequences of $f_1 = \cdots = f_m = 0$ eliminating the variables $x_1$, \ldots, $x_l$. The generator of~$J_{n-1}$ is called the \textit{eliminant} of~$J$ in the variable~$x_n$.

As is well-known, the best way for finding elimination ideals is computing a Gr\"{o}bner basis of $J$ with respect to the pure lexicographic ordering $x_1 > \cdots > x_n$. However, in many cases this computation is practically impossible because of bounded computer resources. For these cases some roundabout ways, such as theory of resultants, should be applied.

Let $f_1 = \sum\limits_{i=0}^p a_ix_1^i$ and $f_2 = \sum\limits_{i=0}^q b_ix_1^i$, where $a_i, b_i \in \mathbb C[x_2, \ldots, x_n]$, be polynomials in~$\mathbb C[x_2, \ldots, x_n][x_1]$. The \textit{Sylvester matrix} $\Syl_{x_1}(f_1, f_2) = (s_{kl})_{k, l = 1}^{p+q}$ of~$f_1$ and~$f_2$ with respect to~$x_1$ is defined by
\[
\begin{split}
s_{j, j+i} &= a_{p-i},\\
s_{p+q-j+1, j+i} &= b_{q-i},
\end{split}
\]
and all other entries of $\Syl_{x_1}(f_1, f_2)$ are equal to zero. The \textit{resultant} of $f_1$ and $f_2$ with respect to~$x_1$, denoted $\mathcal R_{x_1}(f_1, f_2)$, is the determinant of the Sylvester matrix, i.e. $\mathcal R_{x_1}(f_1, f_2) = \det \Syl_{x_1}(f_1, f_2)$. The following lemma states the relation between resultants and elimination ideals.
\begin{lemma}[\cite{CLS}]
\label{lem:relation}
Let $f_1, f_2 \in \mathbb C[x_1, \ldots, x_n]$ have positive degree in~$x_1$. Then, $\mathcal R_{x_1}(f_1, f_2) \in J_1$.
\end{lemma}

\subsection{Thirty-sixth degree univariate polynomial}
At this step, we have derived the following polynomial system in the variables $\tilde{u}_2$, $\tilde{u}_3$ and~$s$:
\begin{equation}
\label{eq:h2h3hmix}
\begin{split}
\tilde{h}^{(2)} &= [4]s^2\tilde{u}_2^6 + [5]s\tilde{u}_2^5 + [6]\tilde{u}_2^4 + [6]\tilde{u}_2^3 + [6]\tilde{u}_2^2 + [6]\tilde{u}_2 + [6],\\
\tilde{h}^{(3)} &= [4]s^2\tilde{u}_3^6 + [5]s\tilde{u}_3^5 + [6]\tilde{u}_3^4 + [6]\tilde{u}_3^3 + [6]\tilde{u}_3^2 + [6]\tilde{u}_3 + [6],\\
\tilde{h}^\text{mix} &= [3]s\tilde{u}_2^2\tilde{u}_3^2 + [4]\tilde{u}_2^2\tilde{u}_3 + [4]\tilde{u}_2\tilde{u}_3^2 + [4]\tilde{u}_2^2 + [4]\tilde{u}_2\tilde{u}_3 + [4]\tilde{u}_3^2 + [4]\tilde{u}_2 + [4]\tilde{u}_3 + [4],
\end{split}
\end{equation}
where $[n]$ means an $n$th degree polynomial in the variable~$s$.

Let us consider an ideal $J = \langle\tilde{h}^{(2)}, \tilde{h}^{(3)}, \tilde{h}^\text{mix} \rangle \subset \mathbb C[\tilde{u}_2, \tilde{u}_3, s]$ and denote its elimination ideals by $J_1 = J \cap \mathbb C[\tilde{u}_3, s]$ and $J_2 = J \cap \mathbb C[s]$.

By direct computation, we find
\begin{equation}
\label{eq:r}
\mathcal R_{\tilde{u}_2}(\tilde{h}^{(2)}, \tilde{h}^\text{mix}) = r (s^2 + 1)^6\, \mathcal S_2^2  \in J_1,
\end{equation}
where $r$ is a 28th total degree polynomial in the variables $\tilde{u}_3$ and~$s$, $\mathcal S_2$ and $\mathcal S_3$ (see below formula~\eqref{eq:S}) are 4th degree polynomials in the variable~$s$ with the known (rather cumbersome) coefficients. For example, the trailing coefficient of~$\mathcal S_k$ is
\begin{multline}
y_{k2}^2 \left[x_{14}x_{k3}(y_{12} - y_{13})(y_{k2} - y_{k4}) - x_{13}x_{k4}(y_{12} - y_{14})(y_{k2} - y_{k3})\right] \\ \times\left[x_{13}(y_{12} - y_{14})(x_{k4}^2 - x_{k3}x_{k4} - y_{k3}y_{k4} + y_{k4}^2)\right. \\\left.+ x_{14}(y_{12} - y_{13})(x_{k3}^2 - x_{k3}x_{k4} - y_{k3}y_{k4} + y_{k3}^2)\right].
\end{multline}
Further,
\begin{equation}
\label{eq:S}
\mathcal R_{\tilde{u}_3}(\tilde{h}^{(3)}, \mathcal R_{\tilde{u}_2}(\tilde{h}^{(2)}, \tilde{h}^\text{mix})) = \mathcal S (s^2 + 1)^{72} s^4 (x_{13} + y_{13}s)^4 (x_{14} + y_{14}s)^4\, \mathcal S_2^{12} \mathcal S_3^{12} \in J_2,
\end{equation}
where $\mathcal S$ is a 36th degree irreducible polynomial in~$s$.

\begin{lemma}
\label{lem:eliminant}
The eliminant of the ideal~$J$ in the variable $s$ is a multiple of the polynomial~$\mathcal S$, i.e. $J_2 = \langle\mathcal S p\rangle$, where $p \in \mathbb C[s]$.
\end{lemma}

\begin{proof}
Let $s_0$ be a root of~$\mathcal S$. Since $s_0$ is not in general a root of the leading coefficients of~$\tilde{h}^{(2)}$ and~$\tilde{h}^{(3)}$ (denoted by $[4]s^2$ in~\eqref{eq:h2h3hmix}), it follows from the Extension Theorem~\cite{CLS} that there exist such $\tilde{u}_{2, 0}$ and $\tilde{u}_{3, 0}$ that $(\tilde{u}_{2, 0}, \tilde{u}_{3, 0}, s_0)$ is a solution to system~\eqref{eq:h2h3hmix}. This exactly means that $s_0$ is a root of the eliminant of~$J$ in the variable~$s$.
\end{proof}

By Lemma~\ref{lem:eliminant}, a root of the polynomial $\mathcal S$ can be extended to a solution of system~\eqref{eq:h2h3hmix}. Moreover, we conjecture that in general every solution of the initial system~\eqref{eq:system} can be expressed in terms of a certain root of the polynomial $\mathcal S$.

\begin{remark}
The polynomial $\mathcal S$ can be found by computing the resultant $\mathcal R_{\tilde{u}_3}(\tilde{h}^{(3)}, r)$ and then dividing it by the polynomial $(s^2 + 1)^{36} s^4 (x_{13} + y_{13}s)^4 (x_{14} + y_{14}s)^4\, \mathcal S_3^{12}$.
\end{remark}

\begin{remark}
\label{rem:realroots}
Since $s = \tan \varphi_2 = \tan \varphi_3$, a complex root of~$\mathcal S$ leads to complex Euler angles having no geometric interpretation. Thus, only real roots of~$\mathcal S$ must be treated.
\end{remark}

\subsection{Structure recovery}
Let $s_0$ be a real root of~$\mathcal{S}$. Then, we can recover the matrices $R^{(2)}$, $R^{(3)}$ and $[t]_\times$ in closed form as follows.

We first propose a simple numerically stable algorithm for finding the $\tilde{u}_2$- and $\tilde{u}_3$-components of the solution. It consists of two steps.  First, one finds all real roots of the \textit{univariate} 6th degree polynomials $\tilde{h}^{(2)}(\tilde{u}_2, s_0)$ and~$\tilde{h}^{(3)}(\tilde{u}_3, s_0)$ defined in~\eqref{eq:tildehj}. Then, the solution $(\tilde{u}_{2, 0}, \tilde{u}_{3, 0}, s_0)$ corresponds to a minimal value of $|\tilde{h}^\text{mix}(\tilde{u}_2, \tilde{u}_3, s_0)|$, where $\tilde{u}_2$ and $\tilde{u}_3$ run over the obtained roots.

Now we find the values
\[
u_{j, 0} = \tilde{u}_{j, 0}/2 - \sign(\tilde{u}_{j, 0})\sqrt{(\tilde{u}_{j, 0}/2)^2 + 1}, \quad j = 2, 3,
\]
subject to $|u_{j, 0}| \leq 1$. After that, we obtain $v_{j, 0}$ by~\eqref{eq:vjformula} and $w_{j, 0}$ by
\begin{equation}
\label{eq:wjformula}
w_{j, 0} = -\frac{a^{(j)}_1 c^{(j)}_2 -  c^{(j)}_1 a^{(j)}_2}{a^{(j)}_1 b^{(j)}_2 -  b^{(j)}_1 a^{(j)}_2},
\end{equation}
where $a^{(j)}_i$, $b^{(j)}_i$, $c^{(j)}_i$ are defined in~\eqref{eq:fjiform}. We also compute the values $u'_{j, 0}$, $v'_{j, 0}$ and $w'_{j, 0}$ by~\eqref{eq:invtrans}. After that, we find the entries of $R^{(j)}$ and $[t]_\times$ by~\eqref{eq:cayley} and~\eqref{eq:transl} respectively.

Let us forget for the moment about the third camera and denote by $R^+ = R^{(2)}(u_{2, 0}, v_{2, 0}, w_{2, 0})$, $R^- = R^{(2)}(u'_{2, 0}, v'_{2, 0} , w'_{2, 0})$, $\mathbf{t}^\pm = \pm\begin{pmatrix}t_x & t_y & t_z\end{pmatrix}^{\mathrm T}$. Then, as is well-known, there are four possible relative positions and orientations for the second camera: $(R^+ \mid \mathbf{t}^+)$, $(R^+ \mid \mathbf{t}^-)$, $(R^{-} \mid \mathbf{t}^{+})$ and $(R^{-} \mid \mathbf{t}^{-})$. Let the true configuration correspond to $(R^{(2)} \mid \mathbf{t})$. Since the point $P_1$ must be in front of all the cameras, it follows that, first, $z_{P_1} > 0$ and, second,
\[
(-x_{O_2})R^{(2)}_{13} + (-y_{O_2})R^{(2)}_{23} + (z_{P_1} - z_{O_2})R^{(2)}_{33} > 0.
\]
Denote by $c_1 = t_x/R^{+}_{13}$, $c_2 = c_1 R^{+}_{33}$. Then,
\begin{itemize}
\item
if $c_1<0$ and $c_2<t_z$, then $R^{(2)} = R^{+}$, $\mathbf{t} = \mathbf{t}^{+}$;
\item
else if $c_1>0$ and $c_2>t_z$, then $R^{(2)} = R^{+}$, $\mathbf{t} = \mathbf{t}^{-}$;
\item
else if $t_x/R^{-}_{13}<0$ and $t_xR^{-}_{33}/R^{-}_{13}<t_z$, then $R^{(2)} = R^{-}$, $\mathbf{t} = \mathbf{t}^{+}$;
\item
else $R^{(2)} = R^{-}$, $\mathbf{t} = \mathbf{t}^{-}$.
\end{itemize}

Similarly, we find the true relative position and orientation $(R^{(3)} \mid \sigma\mathbf{t})$ for the third camera. Here $\sigma$ is either $+1$ or~$-1$.

After we have found $R^{(2)}$, $R^{(3)}$ and $\mathbf{t}$, the coordinates of $O_2$, $O_3$ and~$P_i$ can be recovered as follows (recall that the baseline length $d = l_{O_1O_2}$ fixes the overall scale, see Remark~\ref{rem:ambiguity}):
\begin{align}
\label{eq:coordsO2}
z_{O_2} &= t_z\, \frac{d}{\|\mathbf{t}\|}, & \quad
x_{O_2} &= t_x\, \frac{d}{\|\mathbf{t}\|}, & \quad
y_{O_2} &= t_y\, \frac{d}{\|\mathbf{t}\|},\\
\label{eq:coordsO3}
z_{O_3} &= \frac{x'_{31} (t_x - t_z x'_{21})}{x'_{21} (t_x - t_z x'_{31})}\, z_{O_2}, & \quad
x_{O_3} &= t_x\, \frac{z_{O_3}}{t_z}, & \quad
y_{O_3} &= t_y\, \frac{z_{O_3}}{t_z},\\
\label{eq:coordsPi}
z_{P_i} &= \frac{t_x - t_z x'_{2i}}{x_{1i} - x'_{2i}}\, \frac{z_{O_2}}{t_z}, & \quad
x_{P_i} &= x_{1i} z_{P_i}, & \quad
y_{P_i} &= y_{1i} z_{P_i}.
\end{align}

The true solution to Problem~\ref{problem1}, which is assumed to be unique, corresponds to a root of $\mathcal S$ that minimizes the \textit{reprojection error}
\begin{equation}
\label{eq:epsilon}
\varepsilon = \sum\limits_{j=1}^3\sum\limits_{i=1}^4 (x_{ji} - \hat{x}_{ji})^2 + (y_{ji} - \hat{y}_{ji})^2,
\end{equation}
where the perfectly matched points $\begin{pmatrix}\hat{x}_{ji} & \hat{y}_{ji} & 1\end{pmatrix}^{\mathrm T}$ are defined by
\begin{equation}
\label{eq:hatxy}
\hat{\omega}_{ji}\begin{pmatrix}\hat{x}_{ji} \\ \hat{y}_{ji} \\ 1\end{pmatrix} = R^{(j)\mathrm T} \begin{pmatrix}x_{P_i} - x_{O_j} \\ y_{P_i} - y_{O_j} \\ z_{P_i} - z_{O_j}\end{pmatrix}.
\end{equation}

As a result, we have obtained a unique solution to Problem~\ref{problem1} which, first, satisfies the epipolar constraints~\eqref{eq:epipolar} and, second, minimizes the reprojection error~\eqref{eq:epsilon}. We must yet multiply the obtained coordinates~\eqref{eq:coordsO2},~\eqref{eq:coordsO3} and~\eqref{eq:coordsPi} by $\rho_1^{\mathrm T}$, where $\rho_1$ is defined in~\eqref{eq:rotations}, in order to return to the initial coordinate system. Finally, we note that the matrix $\rho_1^{\mathrm T}R^{(j)}\rho_j$ encodes an information on the initial $j$th camera orientation, e.g. the last column of $\rho_1^{\mathrm T}R^{(j)}\rho_j$ is the $z$-axis of $O_j xyz$.

\section{Experiments on synthetic data}
\label{sec:experiments}
In this section, we compare our algorithm with the five-point Li-Hartley solver~\cite{LH} on the following two sets of synthetic data:
\begin{enumerate}
\item
generic configuration: all simulated scene points are between the planes $z = 1$ and $z = 2$;
\item
planar configuration: all simulated scene points are on the plane $z = 2$.
\end{enumerate}

The third camera center $O_3$ varies randomly between $\frac{1}{3}O_1O_2$ and $\frac{2}{3}O_1O_2$. The baseline length $d = 0.3$ is the same for both sets of data. The field of view equals 45 degrees. For each configuration we add the Gaussian image noise with a standard deviation of one pixel in a $512\times 512$ pixel image. Each experiment is run for 100 trials.

\begin{figure}[t]
\centering
\subfigure[Generic configuration]
{\includegraphics[width=0.46\hsize]{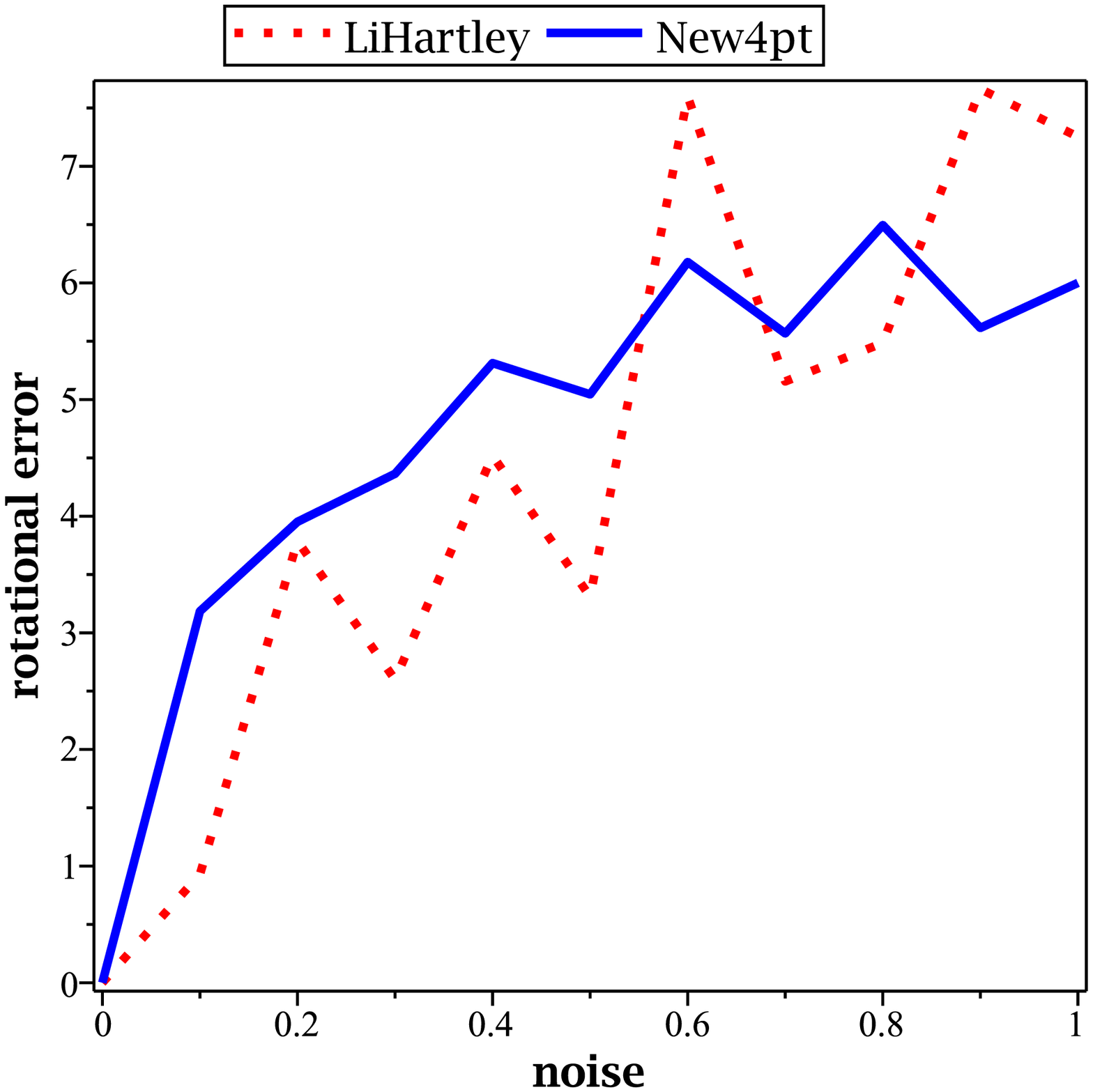}\label{fig:rot_generic}} \qquad
\subfigure[Planar configuration]
{\includegraphics[width=0.46\hsize]{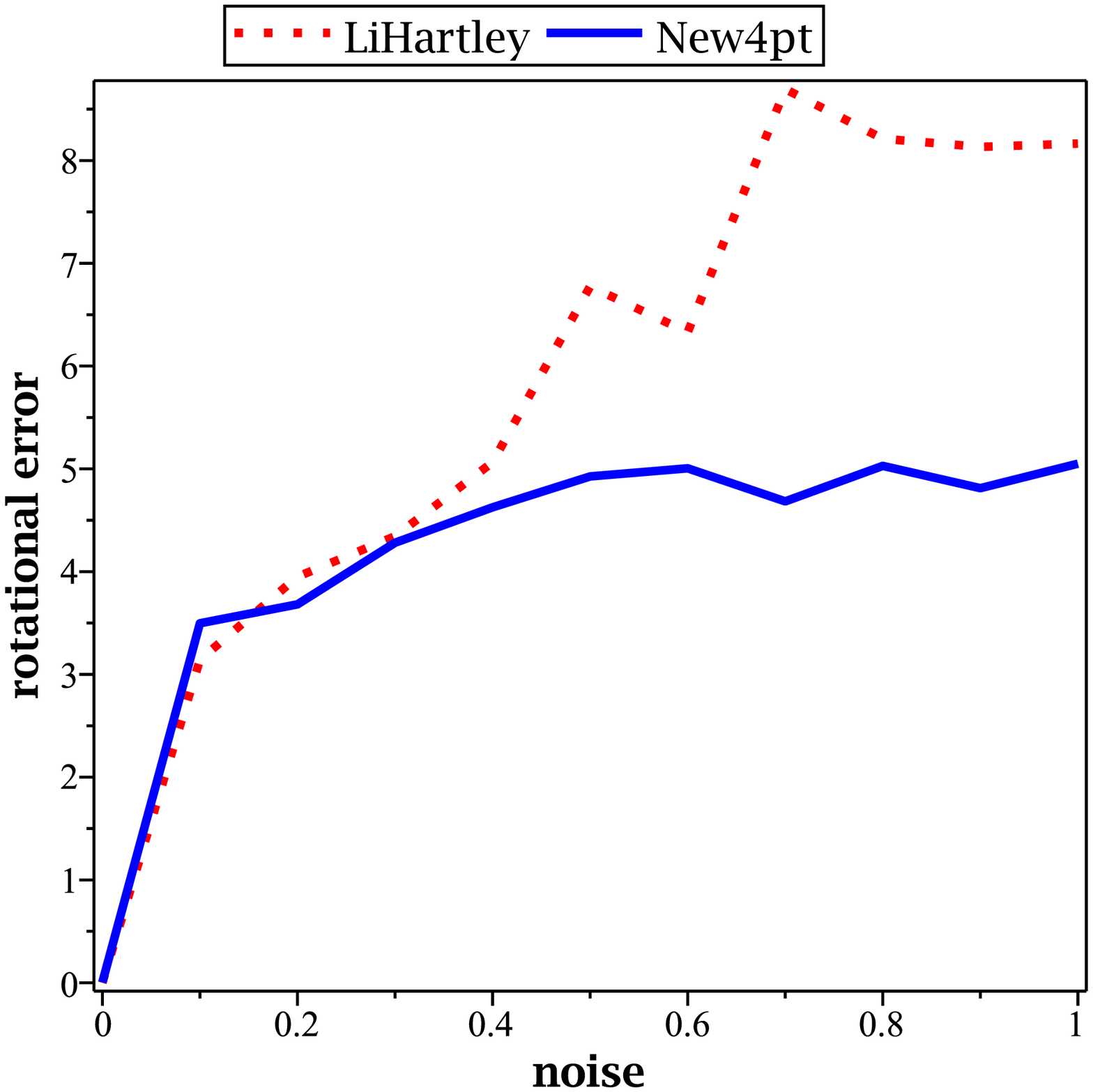}\label{fig:rot_planar}}
\caption{Rotational error (in degrees) changing as a function of image noise (in pixels)}
\label{fig:rot_errors}
\end{figure}

\begin{figure}[t]
\centering
\subfigure[Generic configuration]
{\includegraphics[width=0.46\hsize]{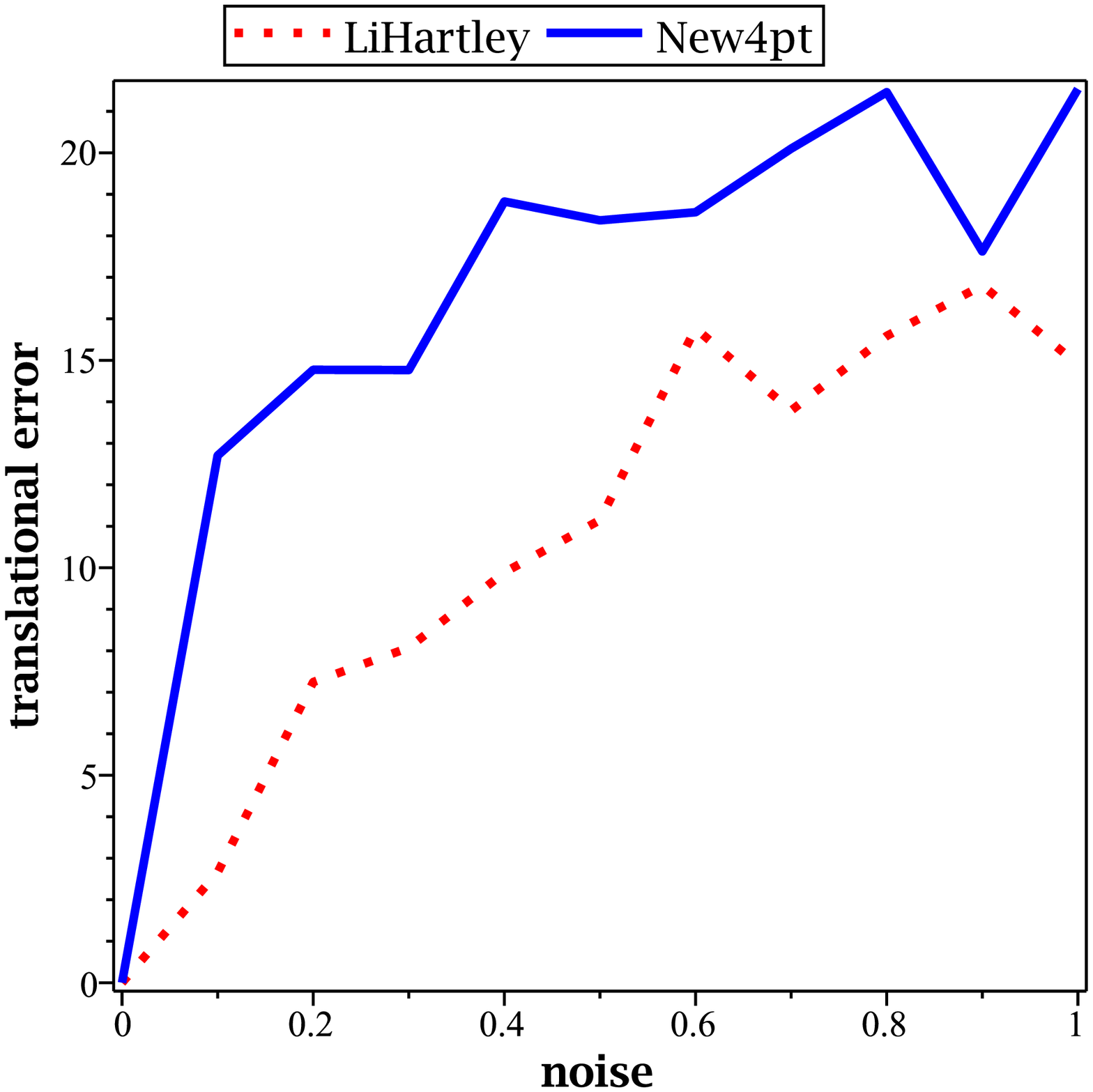}\label{fig:transl_generic}} \qquad
\subfigure[Planar configuration]
{\includegraphics[width=0.46\hsize]{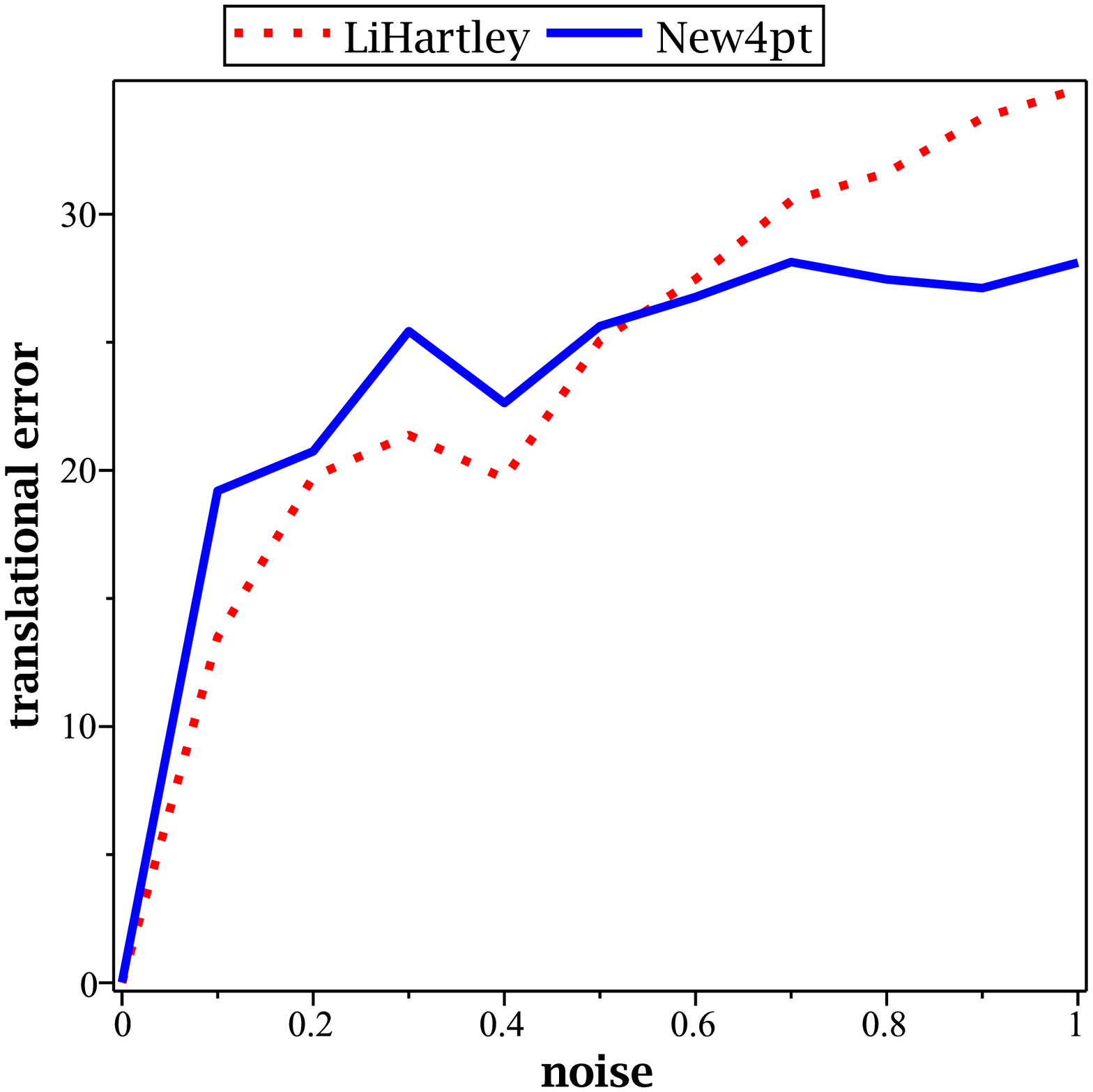}\label{fig:transl_planar}}
\caption{Translational error (in degrees) changing as a function of image noise (in pixels)}
\label{fig:transl_errors}
\end{figure}

The $j$th \textit{rotational error} is defined by
\begin{equation}
\varepsilon^{(j)}_\text{rot} = \arccos \frac{\Tr(\bar{R}^{(j) \mathrm T} \rho_1^{\mathrm T} R^{(j)} \rho_j) - 1}{2},
\end{equation}
where $\bar{R}^{(j)}$ is the true orientation matrix for the $j$th camera, $\Tr(M)$ is the trace of matrix~$M$. The \textit{translational error} is defined by
\begin{equation}
\varepsilon_\text{transl} =  \arccos\frac{\bar{\mathbf{t}}^{\mathrm T} \rho_1^{\mathrm T} \mathbf{t}}{\|\mathbf{t}\| \cdot \|\bar{\mathbf{t}}\|},
\end{equation}
where $\bar{\mathbf{t}}$ is the true translation vector.

The average values of rotational ($= (\varepsilon^{(2)}_\text{rot} + \varepsilon^{(3)}_\text{rot})/2$) and translational errors are reported in Figure~\ref{fig:rot_errors} and Figure~\ref{fig:transl_errors} respectively.

\section{Discussion of results}
\label{sec:discussion}

A non-iterative solution to the four-point three-views pose problem has been proposed for the case of collinear cameras. A computation on synthetic data confirms its correctness and robustness. The new algorithm can be used as a hypothesis-generator for RANSAC-like schemes. Its advantages are
\begin{itemize}
\item
minimal possible number of scene points needed for reconstruction;
\item
uniqueness of the solution;
\item
good enough behavior under image noise conditions even in case of planar scenes.
\end{itemize}

A big number of arithmetic operations needed to derive the polynomial $\mathcal S$ and consequently a big computational error is a weakness of our method in its current stage.

\bibliographystyle{amsplain}

\end{document}